\documentclass[11pt, a4paper, oneside]{article}

\usepackage[margin=1.3in]{geometry}

\def\citep{\cite}

\usepackage{times}
\usepackage{amsfonts,psfrag,epsfig}
\usepackage{amsfonts,epsfig}
\usepackage{color}
\usepackage{amsmath,amssymb,hyperref}
\usepackage[amsmath,thmmarks]{ntheorem}
\usepackage{tabularx}
\usepackage{graphicx}
\usepackage{caption}
\usepackage{subcaption}

\usepackage{tikz}
\usetikzlibrary{intersections}
\usetikzlibrary{shapes,arrows}

\hypersetup{colorlinks=true}
\newtheorem{theorem}{Theorem}
\newtheorem{lemma}[theorem]{Lemma}
\newtheorem{corollary}[theorem]{Corollary}

\newtheorem{claim}[theorem]{Claim}

{ \theoremstyle{nonumberplain}\theoremsymbol{\ensuremath{\Box}}
\theorembodyfont{\normalfont}
\newtheorem{proof}{Proof.}

\theoremsymbol{\ensuremath{\Box}}
\theoremheaderfont{\normalfont\itshape}
\theorembodyfont{\normalfont} \theoremstyle{empty}

}

\newcommand{\union}{\cup}

\newcommand{\beq}{\begin{eqnarray}}
\newcommand{\eeq}{\end{eqnarray}}
\newcommand{\beqn}{\begin{equation}}
\newcommand{\eeqn}{\end{equation}}

\renewcommand{\hat}{\widehat}

\tikzstyle{decision} = [diamond,
                        draw,
                        fill=blue!20,
                        text width=1.5em,
                        text badly centered,
                        node distance=2cm,
                        inner sep=0pt]
\tikzstyle{block} = [rectangle,
                     draw,
                     fill=blue!20,
                     text centered,
                     rounded corners,
                     minimum height=2em]
\tikzstyle{line} = [draw,
                    -latex']
\tikzstyle{cloud} = [draw,
                     ellipse,
                     fill=red!20,
                     node distance=2cm,
                     minimum height=2em]

\begin{document} 

\title{Convergence and Correctness of Max-Product Belief Propagation for Linear Programming\footnote{The preliminary conference version of this work was presented at
Conference on Uncertainty in Artificial Intelligence (UAI), 2015}}
\date{}
\author{
Sejun Park  \and 
Jinwoo Shin\thanks{S.\ Park and J.\ Shin are
with School of Electrical Engineering, Korea Advanced Institute of Science \& Technology, Republic of Korea.
Emails: \{sejun.park, jinwoos\}@kaist.ac.kr}}


\maketitle




\begin{abstract}
The max-product {belief propagation} (BP) is 
a popular message-passing heuristic for approximating a maximum-a-posteriori (MAP) assignment
in a joint distribution represented by a graphical model (GM). 
In the past years, it has been shown that BP can solve
a few classes of linear programming (LP) formulations to combinatorial optimization problems
including maximum weight matching, shortest path and network flow, i.e., BP can
be used as a message-passing solver for certain combinatorial optimizations. However, those LPs and corresponding BP analysis are very sensitive
to underlying problem setups, and it has been not clear what extent these results can be generalized to.
In this paper, we obtain a generic criteria that BP converges to the optimal solution of given LP, and
show that it is satisfied in LP formulations associated to many classical combinatorial optimization problems including 
maximum weight perfect matching, shortest path, traveling salesman, cycle packing, vertex/edge cover and network flow.
\end{abstract}

\section{Introduction}

The max-product belief propagation (BP) is the most popular heuristic for approximating a maximum-a-posteriori (MAP) assignment\footnote{In general,
MAP is NP-hard to compute exactly \citep{chandrasekaran08com}.} of given
graphical model (GM) \citep{yedidia2005constructing,richardson2008modern,mezard2009information,wainwright2008graphical}, where
its performance has been not well understood in loopy GMs, i.e., GM with cycles. 
Nevertheless, BP often shows remarkable performances even on loopy GM.
Distributed implementation, associated ease of programming
and strong parallelization potential are the main reasons for the growing popularity of the BP
algorithm. For example, several software architectures for implementing parallel BPs were recently proposed \citep{LowGKBGH10graph, gonzalez2010parallel, ma2012task}.


In the past years, there have been made extensive research efforts to understand BP performances on loopy GMs 
under connections to combinatorial optimization \cite{bayati2005maximum, sanghavi2011belief, huang2007loopy, salez2009belief, bayati2011belief, shin2013graphical, ruozzi2008st, gamarnik2012belief, chandrasekaran2011counting, bandyopadhyay2006counting, sanghavi2009message}.
In particular, it has been studied about the BP convergence to the correct answer 
under a few classes of loopy GM formulations of combinatorial optimization
problems: matching \citep{bayati2005maximum, sanghavi2011belief, huang2007loopy, salez2009belief}, perfect matching \citep{bayati2011belief}, 
matching with odd cycles \citep{shin2013graphical}, shortest path \citep{ruozzi2008st} and network flow \cite{gamarnik2012belief}.
 The important common feature of these instances is that BP converges to a correct MAP assignment if
linear programming (LP) relaxation of the MAP inference problem is tight, i.e., it has
no integrality gap. In other words, BP can be used as an efficient distributed solver for those LPs, and is presumably a better choice than
classical centralized LP solvers such as simplex methods \citep{dantzig1998linear}, interior point methods \citep{dantzig2006linear} and ellipsoid methods \citep{khachiyan1980polynomial}
for large-scale inputs.
However, these theoretical results on BP are very sensitive to underlying structural properties depending on specific problems
and it is not clear what extent they can be generalized to, e.g.,
the BP analysis for matching problems \citep{bayati2005maximum, sanghavi2011belief, huang2007loopy, salez2009belief}
does not extend to even for perfect matching ones \citep{bayati2011belief}.
In this paper, we overcome such technical difficulties for
enhancing the power of BP as a LP solver. 



\subsection{Contribution}
We establish a generic criteria for GM formulations of given LP so that
BP converges to the optimal LP solution given arbitrary initialization. 
Consequently, it also provides a sufficient condition for guaranteeing that
a BP fixed point is unique.
As one can naturally expect given prior results,
one of our conditions requires the LP tightness. Our main contribution is finding
other sufficient generic conditions so that BP converges to the correct MAP assignment of GM.
First of all, our generic criteria can rediscover all prior BP results on this line,
including matching \citep{bayati2005maximum, sanghavi2011belief,huang2007loopy}, perfect matching \citep{bayati2011belief}, 
matching with odd cycles \citep{shin2013graphical} and shortest path \citep{ruozzi2008st}, i.e., 
we provide a unified framework on establishing the convergence and correctness of BPs in relation to associated LPs.
Furthermore, we provide new instances under our framework: we show that BP can solve LP formulations associated to other popular
combinatorial optimizations including perfect matching with odd cycles,
traveling salesman, cycle packing, network flow and vertex/edge cover, which are not known in the literature.
Here, we remark that the same network flow problem was already studied using BP by Gamarnik et al. \cite{gamarnik2012belief}.
However, 
our BP is different from theirs and much simpler
to implement/analyze:  the authors study BP on continuous GMs,
and we do BP on discrete GMs.
While most prior known BP results on this line focused on the case when the associated LP has an integral solution,
the proposed criteria naturally guides the BP design to compute fractional LP solutions as well (see Section \ref{subsec:matching} 
and Section \ref{subsec:vc} for details).

Our proof technique is built on that of \cite{sanghavi2011belief} where
the authors construct
an alternating path in the computational tree induced by BP to analyze its performance
for the maximum weight matching problem. Such a trick needs specialized case studies depending on the associated LP 
when
the path reaches a leaf of the tree, and this is one of main reasons why it is not easy to generalize
to other problems beyond matching. 
The main technical contribution of this paper is providing a way to avoid the issue in the BP analysis
via carefully analyzing associated LP polytopes.
The main appeals of our results are providing not only tools on BP analysis, but also 
guidelines on BP design for its high performance, i.e.,
one can carefully design a BP given LP so that it satisfies the proposed criteria.
Our results provide not only new tools on BP analysis and design,
but also new directions on efficient distributed (and parallel) solvers for large-scale LPs
and combinatorial optimization problems.


\subsection{Organization}
In Section \ref{sec:pre}, we introduce necessary backgrounds for the BP algorithm. 
We provide the main result of the paper in Section \ref{sec:main}
and its several concrete applications to popular combinatorial optimizations are presented in
Section \ref{sec:applications}.
The proof of the main theorem is presented in Section \ref{sec:mainpf}. 

\section{Preliminaries}\label{sec:pre}

\subsection{Graphical Model} 

A joint distribution of $n$ (binary) random variables $Z=[Z_i]\in \{0,1\}^n$ is called a graphical model (GM) if it factorizes as follows: for $z=[z_i]\in \Omega^n$,
\begin{equation*}
	\Pr[Z=z]~\propto~\prod_{i\in\{1,\dots,n\}}\psi_i(z_i)\prod_{\alpha\in F} \psi_{\alpha} (z_\alpha),\label{eq:generic_gm}
\end{equation*}
where $\{\psi_i,\psi_{\alpha}\}$ are (given) non-negative functions, the so-called factors; 
$F$ is a collection of subsets 
$$F=\{\alpha_1,\alpha_2,...,\alpha_k\}\subset 2^{\{1,2,\dots, n\}}$$
(each $\alpha_j$ is a subset of $\{1,2,\dots, n\}$ with $|\alpha_j|\ge 2$); $z_\alpha$ is the projection of $z$ onto
dimensions included in $\alpha$.\footnote{For example, if $z=[0,1,0]$ and $\alpha=\{1,3\}$, then $z_\alpha=[0,0]$.} 
In particular, $\psi_i$ is called a variable factor.
Figure~\ref{fig:startup} depicts
the graphical relation between factors $F$ and variables $z$.

\vspace{0.1in}
\begin{figure}[h]
\centering
\includegraphics[width=0.3\textwidth]{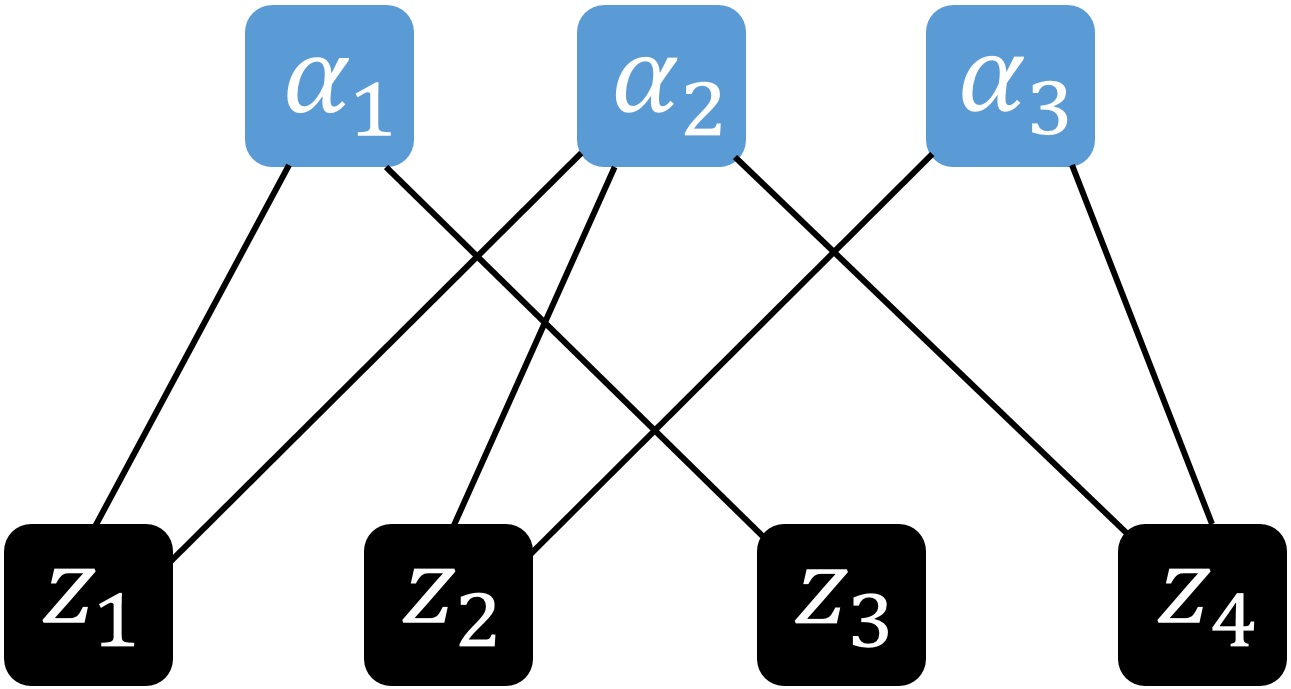}

\caption{
%
Factor graph for the graphical model with $F=\{\alpha_1,\alpha_2,\alpha_3\}$ and $n=4$:
$\Pr[z] \propto \psi_{\alpha_1}(z_1,z_3)\psi_{\alpha_2}(z_1,z_2,z_4)\psi_{\alpha_3}(z_2,z_3,z_4).$
Each $\alpha_j$ selects a subset of $z$, e.g., $\alpha_1$ selects
$\{z_1,z_3\}$.
%
%
}
\label{fig:startup}
\normalsize
\end{figure}
%
%
%
Assignment ${z}^*$ is called a maximum-a-posteriori (MAP) assignment if ${z}^*$ satisfies
${z}^*=\arg\max_{{z}\in\{0,1\}^n} \Pr[{z}].$
This means that computing a MAP assignment requires comparing
$\Pr[{z}]$ for all possible $z$, which is typically
computationally intractable (i.e., NP-hard) unless the induced bipartite graph 
of factors $F$ and variables $z$, so-called factor graph, has a bounded treewidth \citep{chandrasekaran08com}.
%
%
%
%

\subsection{Max-Product Belief Propagation}
The (max-product) belief propagation (BP) algorithms are popular heuristics for approximating the MAP assignment in a graphical model.
BP is an iterative procedure; at each iteration $t$, there are four 
messages 
$$\{m^{t}_{\alpha\rightarrow i}(c), 
m^{t}_{i\rightarrow\alpha}(c):
c\in\{0,1\} \}$$
 between
every variable $z_i$ and every associated $\alpha\in F_i$, where
 $F_i:= \{\alpha\in F: i \in \alpha\}$; that is, $F_i$ is
a subset of $F$ such that all $\alpha$ in $F_i$ include the $i^{th}$
position of $z$ for any given $z$.
Then, messages are updated as follows:
\begin{align}
&\quad m^{t+1}_{\alpha\rightarrow i}(c) ~=~ \max_{z_{\alpha}:z_i=c}
                                   \psi_\alpha (z_{\alpha})
\prod_{j\in \alpha\setminus i} m_{j\rightarrow \alpha}^t (z_j)\label{eq:msg:alpha_to_i}\\
&\quad m^{t+1}_{i\rightarrow\alpha}(c) ~=~
\psi_i(c)\prod_{\alpha^{\prime}\in F_i\setminus \alpha} m_{\alpha^{\prime}
\rightarrow i}^t (c)\label{eq:msg:i_to_alpha}.
\end{align}
%

First, we note that each $z_i$ only sends messages to $F_i$; that is,
$z_i$ sends messages to $\alpha_j$ only if $\alpha_j$ selects/includes $i$.
%
%
%
The outer-term in the message computation \eqref{eq:msg:alpha_to_i} 
is maximized over all possible $z_\alpha\in\{0,1\}^{|\alpha|}$ with $z_i=c$.
%
%
The inner-term is a product that only depends on the variables $z_j$ (excluding $z_i$) that
are connected to $\alpha$. 
%
%
%
%
The message-update \eqref{eq:msg:i_to_alpha} from a variable $z_i$ to a factor $\psi_\alpha$ is a product
which considers all messages received by $\psi_\alpha$ in the previous iteration,
except for the message sent by $z_i$ itself.
%
%

One can reduce the complexity of messages by combining 
\eqref{eq:msg:alpha_to_i} and \eqref{eq:msg:i_to_alpha} as:
\begin{align*}
m^{t+1}_{i\rightarrow\alpha}(c) =
\psi_i(c)
\prod_{\alpha^{\prime}\in F_i\setminus \alpha} \max_{z_{\alpha^\prime}:z_i=c}
                                   \psi_{\alpha^\prime} (z_{\alpha^\prime})
\prod_{j\in \alpha^\prime\setminus i} m_{j\rightarrow \alpha^\prime}^t (z_j),
\end{align*}
which we analyze in this paper.
Finally, given a set of messages $\{m_{i\to\alpha}(c)$, $m_{\alpha\to
i}(c):c\in\{0,1\}\}$, the so-called BP marginal beliefs are computed as follows:
\begin{eqnarray}\label{eq:bpdecision}
b_i[z_i]&=&
\prod_{\alpha\in F_i} m_{\alpha\to i}(z_i).\label{eq:marginalbelief}
%
%
%
%
\end{eqnarray}
Then, the BP algorithm outputs $z^{BP}=[z_i^{BP}]$ as
$$
z_i^{BP}=\begin{cases}
1&\mbox{if}~ b_i[1]>b_i[0]\\
?&\mbox{if}~b_i[1]=b_i[0]\\
0&\mbox{if}~ b_i[1]<b_i[0]
\end{cases}.
$$
It is known that $z^{BP}$ converges to a MAP assignment 
if the factor graph 
is a tree and the MAP assignment is unique. However, if the graph has loops in it, the BP
algorithm has no guarantee to find a MAP assignment in general.
%

\section{Convergence and Correctness of Belief Propagation} 
\label{sec:main}
\subsection{Convergence and Correctness Criteria of BP}
In this section, we provide the main result of this paper: a convergence and correctness criteria of BP.
Consider the following GM:  for $x=[x_i]\in \{0,1\}^n$ and $w=[w_i]\in \mathbb R^n$,
\begin{equation}
	\Pr[X=x]~\propto~\prod_{i} e^{-w_i x_i}\prod_{\alpha\in F} \psi_{\alpha} (x_\alpha),\label{eq:gm1}
\end{equation}
where $F$ is the set of non-variable factors and
the factor function $\psi_\alpha$ for $\alpha\in F$ is defined as
\begin{align*}
&\psi_{\alpha}(x_{\alpha}) = 
\begin{cases}
1&\mbox{if}~ A_{\alpha} x_{\alpha}\ge b_{\alpha}\\
0&\mbox{otherwise}
\end{cases},
\end{align*}
for some
matrices $A_{\alpha}$ and vectors $b_{\alpha}$. 
Now we consider the linear programming (LP) corresponding the above GM:
\begin{equation}\label{eq:lp1}
\begin{split}
	&\mbox{minimize}\qquad~ w\cdot x\\
	&\mbox{subject to}\qquad A_\alpha x\ge b_\alpha~\text{for all $\alpha\in F$}\\
    &\qquad\qquad\qquad~ x=[x_i]\in [0,1]^n
\end{split}
\end{equation}
To simplify the notation, we often use
$Ax\ge b$ with $A\in\mathbb{R}^{m\times n},b\in\mathbb{R}^m$ which includes all inequalities $A_\alpha x\ge b_\alpha$ and $x\in[0,1]^n$.
{Without loss of generality, we assume that $\|A_{i*}\|_2=1$ for all $i=1,2,\dots,m$ throughout this paper, 	where $A_{i*}$ is the $i$-th row of $A$.
Similarly, we denote $A_{*i}$ as the $i$-th column of $A$.}
One can easily observe that the MAP assignments for GM \eqref{eq:gm1} corresponds to the (optimal) solution of LP \eqref{eq:lp1}
if the LP has an integral solution $x^*\in \{0,1\}^n$.
Furthermore, if the solution of LP \eqref{eq:lp1} is unique, the there exists a positive constant $\rho$ satisfying the following identity:
\begin{equation*}
		\rho:=\inf_{x\in\mathcal{P}\setminus x^*} \frac{w\cdot x- w\cdot x^*}{\|x-x^*\|_1}>0.
\end{equation*}
Using the notation and observation, 
we establish the following sufficient conditions
so that the max-product BP can indeed find the LP solution.

\begin{figure}[t]
    \centering
    \begin{subfigure}[b]{0.2\textwidth}
        \includegraphics[width=\textwidth]{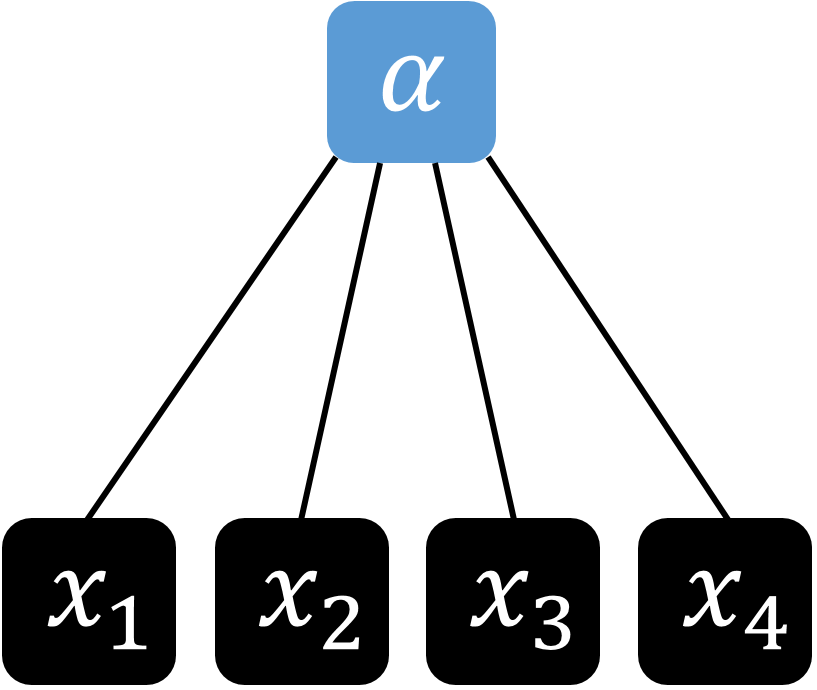}
        \caption{$x^*_\alpha$}
    \end{subfigure}
    \quad
    \begin{subfigure}[b]{0.2\textwidth}
        \includegraphics[width=\textwidth]{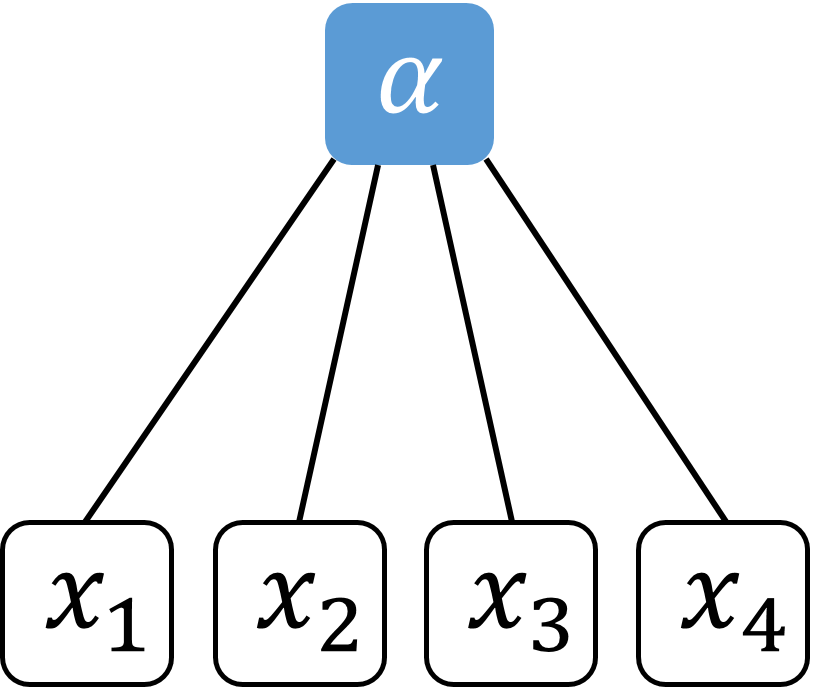}
        \caption{$x_\alpha$}
    \end{subfigure}
    \quad
    \begin{subfigure}[b]{0.2\textwidth}
        \includegraphics[width=\textwidth]{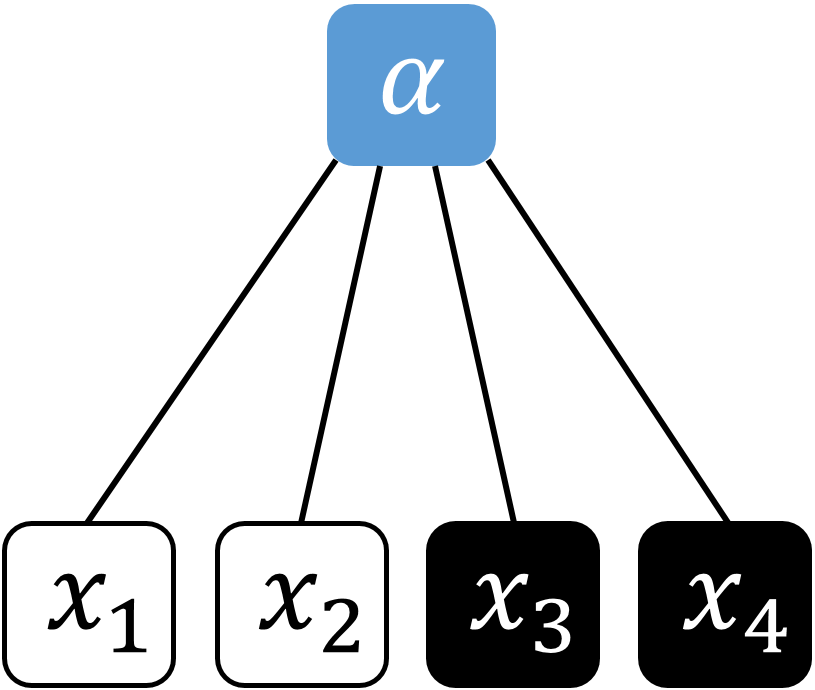}
        \caption{$x^\prime_\alpha$}
    \end{subfigure}
    \quad
    \begin{subfigure}[b]{0.2\textwidth}
        \includegraphics[width=\textwidth]{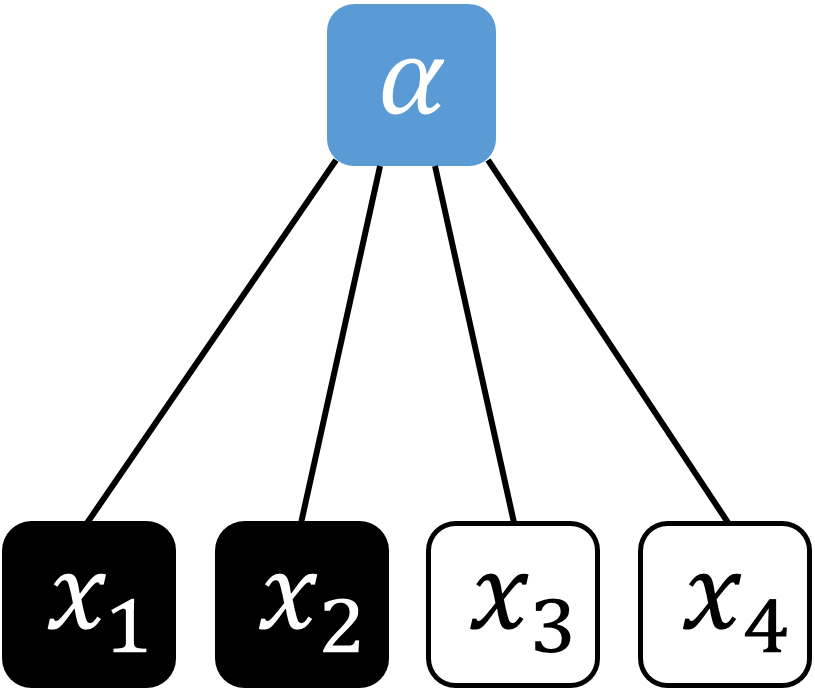}
        \caption{$x^{\prime\prime}_\alpha$}
    \end{subfigure}
    \caption{
    Illustration of Condition {\em C3} of Theorem \ref{thm:main}
    under $i=1,\gamma=\{2\}$ and $\psi_\alpha(x_\alpha)=1$ (i.e., say $\psi_\alpha$ satisfies) if and only if $\sum_{i\in\alpha}x_i=1$.
    All four variables $x^*_\alpha,x_\alpha,x^\prime_\alpha,x^{\prime\prime}_\alpha$ must satisfy $\psi_\alpha$. 
    For example, let 
    $x^*_\alpha=(1,0,0,0)$ and $x_\alpha=(0,1,0,0)$.
    Then, both $x^\prime_\alpha=(x^*_1,x^*_2,x_3,x_4)$ and $x^{\prime\prime}_\alpha=(x_1,x_2,x^*_3,x^*_4)$ satisfy $\psi_\alpha$.}
    \label{fig:C3}
\end{figure}
\begin{theorem}\label{thm:main}
Suppose the following conditions hold:
\begin{itemize}
\item[C1.] LP \eqref{eq:lp1} has a unique integral solution $x^*\in\{0,1\}^n$, i.e., it is tight.
\item[C2.] For every $i\in \{1,2,\dots, n\}$, the number of factors associated with $x_i$ is at most two, i.e., $|F_i|\leq 2.$
\item[C3.] For every factor $\psi_\alpha$, every $x_\alpha\in\{0,1\}^{|\alpha|}$ with $\psi_\alpha(x_\alpha)=1$, and 
	every $i\in\alpha$ with $x_i\neq x^*_i$, 
	there exists 
$\gamma\subset \alpha$ 
such that
$$|\{j \in\{i\}\union \gamma:|F_j|=2\}|\le 2$$
$$
\psi_\alpha(x^\prime_\alpha)=1,\qquad
\mbox{
where 
$x^\prime_k = \begin{cases}
		x_k~&\mbox{if}~k\notin \{i\}\union \gamma\\
		x^*_k~&\mbox{otherwise}
\end{cases}$.}$$
$$
\psi_\alpha(x^{\prime\prime}_\alpha)=1,\qquad
\mbox{
where 
$x^{\prime\prime}_k = \begin{cases}
		x_k~&\mbox{if}~k\in\{i\}\union \gamma\\
		x^*_k~&\mbox{otherwise}
\end{cases}$.}$$
    
\end{itemize}
Then the max-product BP on GM \eqref{eq:gm1} with arbitrary initial message converges to the solution of LP \eqref{eq:lp1} 
in $\left(\frac{w_{\max}}{\rho}+1\right)K$ iterations,
where\footnote{$A_\xi$ is a square matrix consisting of rows of $A$ corresponding to the row index set $\xi$,
and $\mathbf{1}$ is the vector consisting of ones.} 
$$K=\max_{\xi\subset\{1,\dots,m\}:|\xi|=n,\det(A_\xi)\ne 0} \|(A_\xi)^{-1}\mathbf{1}\|_1
\qquad\mbox{and}\qquad
w_{\max}=\max_j |w_j|.
$$
\end{theorem}
Figure \ref{fig:C3} illustrates Condition {\em C3}.
Since Theorem \ref{thm:main} holds for arbitrary initial messages, 
it also provides the uniqueness of BP fixed points, as stated in what follows.
\begin{corollary}\label{cor:fixedpoint}
The max-product BP on GM \eqref{eq:gm1} has a unique fixed point if Conditions C1-C3 hold.
\end{corollary}

\subsection{Remarks on Theorem \ref{thm:main}}
Conditions {\em C2, C3} of Theorem \ref{thm:main} are typically easy to check given GM \eqref{eq:gm1} and
the uniqueness in {\em C1} can be easily guaranteed via adding random noises.
On the other hand, the integral property in {\em C1} requires to analyze LP \eqref{eq:lp1}, where it has been extensively studied
in the field of combinatorial optimization \citep{schrijver2003combinatorial}, e.g.,
Totally Unimodular (TUI). 
However, the conditions of Theorem \ref{thm:main} do not imply TUI, and vice versa.
Since the TUI condition requires all vertices of the LP polytope of constraints are integral,
Condition {\em C1} is, at least, not stronger than it. On the other hand, even if the LP polytope is TUI, 
{\em C3} might be not satisfied in general.
For example, consider the following TUI constraint:
$$\begin{bmatrix}
1 & -1 & 1 & 0 & -1\\
-1 & 1 & -1 & 0 & 1\\
1 & 0 & 0 & 1 & 0\\
-1 & 0 & 0 & -1 & 0
\end{bmatrix}x_\alpha\ge
\begin{bmatrix}
0\\
0\\
1\\
-1
\end{bmatrix}.$$
In Condition {\em C3}, suppose $x_\alpha^*=(1,1,0,0,0)$ and $x_\alpha=(0,0,1,1,1)$.
{Then, one can easily observe that $\psi_\alpha(x_\alpha^*)=\psi_\alpha(x_\alpha)=1$.}
However, if we choose $i=1$, then there does not exist $\gamma$ satisfying Condition {\em C3}.

We also remark that 
for some special cases of GM \eqref{eq:gm1}, e.g.,
entries of $A$ consists of $\{0,\pm 1\}$,
Condition {\em C3} and the constant $K$ can be removed and simplified, respectively,
as stated in the following lemmas.
\begin{lemma}\label{lem:c3}
Given GM \eqref{eq:gm1}, if every factor function $\psi_\alpha(\cdot)$ can be expressed as
$$\psi_\alpha(x_\alpha)=
\begin{cases}
&1\qquad\text{if}~L_\alpha\le a_\alpha\cdot x_\alpha\le U_\alpha\\
&0\qquad\text{otherwise}
\end{cases}$$
for some $a_\alpha\in\{-1,1\}^{|\alpha|}$ and $L_\alpha,U_\alpha\in\mathbb{R}$, then GM \eqref{eq:gm1} satisfies Condition C3 of Theorem \ref{thm:main}.
\end{lemma}
\begin{proof}
As stated in Condition {\em C3}, suppose that there are two assignments $x^*_\alpha\ne x_\alpha$ satisfying $\psi_\alpha(x^*_\alpha)=\psi_\alpha(x_\alpha)=1$ and $x_i\ne x^*_i$ for some $i\in\alpha$.
Since $a_\alpha\in\{-1,1\}^{|\alpha|}$, we have
$$a_\alpha\cdot (x^\dagger_\alpha-x_\alpha)\in\{-1,1\}$$ where $x^\dagger_\alpha$ is defined as
$$x^\dagger_k=\begin{cases}
&x^*_k\qquad\text{if}~k=i\\
&x_k\qquad\text{otherwise}
\end{cases},\qquad\qquad\mbox{for}~~k\in\alpha.$$
There are two cases $a_\alpha\cdot (x^*_\alpha-x_\alpha)<0$ and $a_\alpha\cdot (x^*_\alpha-x_\alpha)\ge 0$.
We assume the case $a_\alpha\cdot (x^*_\alpha-x_\alpha)<0$ while the case $a_\alpha\cdot (x^*_\alpha-x_\alpha)\ge 0$ can be argued in a similar manner.
If $a_\alpha\cdot (x^\dagger_\alpha-x_\alpha)=-1$,
then choosing $\gamma=\emptyset$ results in
\begin{align*}
&L_\alpha\le a_\alpha\cdot x^\prime_\alpha=a_\alpha\cdot x_\alpha-1\le U_\alpha\\
&L_\alpha\le a_\alpha\cdot x^{\prime\prime}_\alpha=a_\alpha\cdot x^*_\alpha+1\le U_\alpha,
\end{align*} 
i.e., $\psi(x^\prime_\alpha)=\psi(x^{\prime\prime}_\alpha)=1$,
which satisfies Condition {\em C3}.
Now, suppose that $a_\alpha\cdot (x^\dagger_\alpha-x_\alpha)=1$.
Since we assumed $a_\alpha\cdot (x^*_\alpha-x_\alpha)<0$, there exists $j\in\alpha$ such that $x_j\ne x^*_j$ and $a_\alpha\cdot (x^\ddagger_\alpha-x_\alpha)=-1$ where $x^\ddagger_\alpha$ is defined as
$$x^\ddagger_k=\begin{cases}
&x^*_k\qquad\text{if}~k=j\\
&x_k\qquad\text{otherwise}
\end{cases},\qquad\qquad\mbox{for}~~k\in\alpha.$$
Then, choosing $\gamma = \{j\}$ results in
\begin{align*}
&L_\alpha\le a_\alpha\cdot x^\prime_\alpha=a_\alpha\cdot x_\alpha\le U_\alpha\\
&L_\alpha\le a_\alpha\cdot x^{\prime\prime}_\alpha=a_\alpha\cdot x^*_\alpha\le U_\alpha,
\end{align*}
i.e., $\psi(x^\prime_\alpha)=\psi(x^{\prime\prime}_\alpha)=1$,
which satisfies Condition {\em C3}.
This completes the proof of Lemma \ref{lem:c3}.
\end{proof}
\begin{lemma}\label{lem:polycomp}
If GM \eqref{eq:gm1} satisfies Conditions C1-C3 and entries of $A$ consists of $\{0,\pm1\}$, then $K\le n^{2.5}$ where $K$ is defined in Theorem \ref{thm:main}.
\end{lemma}
\begin{proof}
For any $n\times n$ invertible submatrix $A_\xi$ of $A$, 
it is known \cite{bolker2006simple} that every entry of $(A_\xi)^{-1}$ is in $\{0,\pm 1/2,\pm 1\}$.
This observation directly leads to
the following bound on $K$:
\begin{align*}
    K&=\max_{\xi\subset\{1,\dots,m\}:|\xi|=n,\det(A_\xi)\ne 0}
    \|(\widetilde A_\xi)^{-1}\mathbf{1}\|_1\\
    &\le \sqrt{n}\times\max_{\xi\subset\{1,\dots,m\}:|\xi|=n,\det(A_\xi)\ne 0}
    \|(A_\xi)^{-1}\mathbf{1}\|_1\\
    &\le n^{2.5}
\end{align*}
where $\widetilde A$ is a row scaled matrix of $A$ so that $\widetilde A_{i*}=c_iA_{i*}$ and $\|\widetilde A_{i*}\|_2=1$ for some constant $c_i$.
This completes the proof of Lemma \ref{lem:polycomp}.
\end{proof}

\section{Applications of Theorem \ref{thm:main}}\label{sec:applications}

In this section, we introduce concrete instances of LPs satisfying the conditions of Theorem \ref{thm:main}
so that BP correctly converges to its optimal solution.
Specifically, we consider LP formulations associated to several
combinatorial optimization problems including
shortest path, maximum weight perfect matching,
traveling salesman, maximum weight disjoint vertex cycle packing, vertex/edge cover and network flow.
We note that the shortest path result, i.e., Corollary \ref{cor:shortestpath},
is known \citep{ruozzi2008st}, where we rediscover it as a corollary of Theorem \ref{thm:main}.
Our other results, i.e., Corollaries \ref{cor:matching}-\ref{cor:networkflow},
are new and what we first establish in this paper.

\subsection{Example I: Shortest Path}\label{sec:shortest}
Given a directed graph $G=(V,E)$ and non-negative edge weights $w=[w_e:e\in E]\in \mathbb R_+^{|E|}$, 
the shortest path problem
 is to find the shortest path from the source $s$ to the destination $t$: 
it minimizes the sum of edge weights along the path. One can naturally design the following LP for this problem:
\begin{equation}\label{lp:shortestpath}
\begin{split}
	&\mbox{minimize}\qquad~ w\cdot x\\
	&\mbox{subject to}\qquad \sum_{e\in \delta^o(v)} x_{e}-\sum_{e\in\delta^i(v)} x_{e}
	=
	\begin{cases}
		&1~~~\mbox{if}~v=s\\
		&-1~\mbox{if}~v=t\\
		&0~~~\mbox{otherwise}
	\end{cases}\\
	&\qquad\qquad\qquad~ x=[x_{e}]\in [0,1]^{|E|}.
\end{split}
\end{equation}

where $\delta^i(v),\delta^o(v)$ are sets of incoming, outgoing edges of $v$.
It is known that the above LP always has an integral solution, i.e., the shortest path from $s$ to $t$.
We consider the following GM for LP \eqref{lp:shortestpath}: 
\begin{equation}\label{gm:shortestpath}
	\Pr[X=x]~\propto~\prod_{e\in E} e^{-w_{e} x_{e}}\prod_{v\in V} \psi_{v} (x_{\delta(v)}),
\end{equation}
where $\delta(v)=\delta^i(v)\cup\delta^o(v)$ and the factor function $\psi_v$ is defined as
\begin{align*}
&\psi_{v}(x_{\delta(v)}) = 
\begin{cases}
1&\mbox{if}~ \sum_{e\in\delta^o(v)} x_{e}-\sum_{e\in\delta^i(v)} x_{e}\\
&\qquad=
	\begin{cases}
		&1~~~\mbox{if}~v=s\\
		&-1~\mbox{if}~v=t\\
		&0~~~\mbox{otherwise}
	\end{cases}\\
0&\mbox{otherwise}
\end{cases}.
\end{align*}
For the above GM \eqref{gm:shortestpath}, one can easily check that Condition {\em C2} 
of Theorem \ref{thm:main} and the condition of Lemma \ref{lem:c3} hold. This directly leads to
the following corollary.
\begin{corollary}\label{cor:shortestpath}
If the shortest path from $s$ to $t$, i.e., the solution of the shortest path LP \eqref{lp:shortestpath}, is unique, then the max-product BP on GM \eqref{gm:shortestpath} converges in $O(w_{\max}|E|^{2.5}/\rho)$ iterations.
\end{corollary}
The uniqueness condition in the above corollary is easy to guarantee by adding small random noises to edge weights.

\subsection{Example II: Maximum Weight Perfect Matching}\label{subsec:matching}
Given an undirected graph $G=(V,E)$ and non-negative edge
weights $w=[w_e:e\in E]\in \mathbb{R}_+^{|E|}$ on edges, the maximum weight perfect matching problem is to find a set of edges such that each vertex is connected to exactly one edge in the set and the sum of edge weights in the set is maximized. One can naturally design the following LP for this problem:
\begin{equation}\label{lp:matching}
\begin{split}
	&\mbox{maximize}\qquad~ w\cdot x\\
	&\mbox{subject to}\qquad \sum_{e\in\delta(v)} x_e= 1\\
	&\qquad\qquad\qquad~ x=[x_e]\in [0,1]^{|E|}.
\end{split}
\end{equation}
where $\delta(v)$ is the set of edges connected to a vertex $v$. 
If the above LP has an integral solution, it corresponds to the solution of the maximum weight perfect matching problem.\\\\
It is known that 
the maximum weight matching LP \eqref{lp:matching}
always has a half-integral solution $x^*\in \{0,\frac12,1\}^{|E|}$. We will design BP for obtaining the half-integral solution.
To this end, duplicate each edge $e$ to $e_1,e_2$ and define a new graph $G^\prime=(V,E^\prime)$ where $E^\prime=\{e_1,e_2:e\in E\}$.
Then, we suggest the following equivalent LP that always have an integral solution:
\begin{equation}\label{lp:modmatching}
\begin{split}
	&\mbox{maximize}\qquad~ w^\prime\cdot x\\
	&\mbox{subject to}\qquad \sum_{e_i\in\delta(v)} x_{e_i}= 2\\
	&\qquad\qquad\qquad~ x=[x_{e_i}]\in [0,1]^{|E^\prime|}.
\end{split}
\end{equation}
where $w^\prime_{e_1}=w^\prime_{e_2}=w_e$. 
One can easily observe that solving LP \eqref{lp:modmatching} is equivalent to solving LP \eqref{lp:matching}
due to our construction of $G^\prime$ and $w^\prime$.
Now, construct the following GM for LP \eqref{lp:modmatching}: 
\begin{equation}\label{gm:modmatching}
	\Pr[X=x]~\propto~\prod_{e_i\in E^\prime} e^{w^\prime_{e_i} x_{e_i}}\prod_{v\in V} \psi_{v} (x_{\delta(v)}),
\end{equation}
where the factor function $\psi_v$ is defined as
\begin{align*}
&\psi_{v}(x_{\delta(v)}) = 
\begin{cases}
1&\mbox{if}~ \sum_{e_i\in\delta(v)} x_{e_i}= 2\\
0&\mbox{otherwise}
\end{cases}.
\end{align*}
For the above GM \eqref{gm:modmatching}, one can easily check that Condition {\em C2} 
of Theorem \ref{thm:main} and the condition of Lemma \ref{lem:c3} hold. This directly leads to
the following corollary.
\begin{corollary}\label{cor:matching}
If the solution of the maximum weight perfect matching LP \eqref{lp:modmatching} is unique, then 
the max-product BP on GM \eqref{gm:modmatching} converges in $O(w_{\max}|E|^{2.5}/\rho)$ iterations. 
\end{corollary}
Again, the uniqueness condition in the above corollary is easy to guarantee by adding small random noises to edge weights $[w^\prime_{e_i}]$.
We note that it is known \citep{bayati2011belief} that BP converges to
the unique and integral solution of LP \eqref{lp:matching}, while Corollary \ref{cor:matching}
implies that BP can solve it without the integrality condition of LP \eqref{lp:matching} by solving GM \eqref{gm:modmatching}.
We note that one can easily obtain a similar result for the maximum weight (non-perfect) matching problem, where
we omit the details in this paper.

\subsection{Example III: Maximum Weight Perfect Matching with Odd Cycles}
In previous section we prove that BP converges to the optimal (possibly, fractional) 
solution of LP \eqref{lp:modmatching}, equivalently LP \eqref{lp:matching}.
One can add odd cycle (also called Blossom) constraints and make those LPs tight i.e. solves the maximum weight perfect matching problem:
\begin{equation}\label{lp:matchingcycle}
\begin{split}
	&\mbox{maximize}\qquad~ w\cdot x\\
	&\mbox{subject to}\qquad \sum_{e\in\delta(v)} x_e= 1,\quad\forall\, v\in V\\
&\qquad\qquad\qquad~ \sum_{e\in C} x_e\le \frac{|C|-1}{2}, \quad\forall C\in\mathcal C,\\
	&\qquad\qquad\qquad~ x=[x_e]\in [0,1]^{|E|}.
\end{split}
\end{equation}
where $\mathcal C$ is a set of odd cycles in $G$.
The authors \citep{shin2013graphical} study BP for solving
LP \eqref{lp:matchingcycle} by replacing $\sum_{e\in\delta(v)} x_e= 1$ by $\sum_{e\in\delta(v)} x_e\leq 1$, i.e., 
for the maximum weight (non-perfect) matching problem.
Using Theorem \ref{thm:main}, one can extend the result to 
the maximum weight perfect matching problem, i.e., solving LP \eqref{lp:matchingcycle}.
To this end, we follow the approach 
\citep{shin2013graphical} and
construct the following graph $G^\prime=(V^\prime,E^\prime)$ and weight $w^\prime=[w^\prime_e:
e\in E^{\prime}]\in\mathbb R^{|E^\prime|}$ given set $\mathcal C$ of disjoint odd cycles: 
\begin{align*}
&V^\prime=V\cup\{v_C:C\in\mathcal{C}\}\\
&E^\prime=\{(u,v_C):u\in C,C\in\mathcal{C}\}\cup E\setminus\{e\in C:C\in\mathcal{C}\}
\end{align*}
\begin{align*}
&w^\prime_e= 
\begin{cases}
\frac{1}{2}\sum_{e^\prime\in E(C)}(-1)^{d_C(u,e^\prime)}w_{e^\prime}&\mbox{if}~e={(u,v_C)}\\
&\mbox{for some}~C\in\mathcal{C}\\
w_e&\mbox{otherwise}
\end{cases},
\end{align*}
where
$d_C(u,e^\prime)$ is the graph distance between $u,e^\prime$ in cycle $C$. 
Then, LP \eqref{lp:matchingcycle} is equivalent to the following LP:
\begin{equation}\label{lp:oddmatching}
\begin{split}
	&\mbox{maximize}\qquad~ w^\prime\cdot y\\
	&\mbox{subject to}\quad~ \sum_{e\in\delta(v)} y_e= 1,\qquad\qquad\qquad\qquad\forall\, v\in V\\
	&\qquad\qquad\quad \sum_{u\in V(C)}(-1)^{d_C(u,e)}y_{(v_C,u)}\in [0,2],~\forall e\in E(C)\\
	&\qquad\qquad\quad \sum_{e\in\delta(v_C)}y_e\le |C|-1, \qquad\qquad\quad~~\forall C\in\mathcal{C}\\
	&\qquad\qquad\qquad y=[y_e]\in [0,1]^{|E^\prime|}.
\end{split}
\end{equation}
Now, we construct the following GM from the above LP:
\begin{equation}\label{gm:oddmatching}
	\Pr[Y=y]~\propto~\prod_{e\in E} e^{w_e y_e}\prod_{v\in V} \psi_{v} (y_{\delta(v)})\prod_{C\in\mathcal{C}} \psi_{C} (y_{\delta(v_C)}),
\end{equation}
where the factor function $\psi_v$, $\psi_C$ is defined as
\begin{align*}
&\psi_{v}(y_{\delta(v)}) = 
\begin{cases}
1&\mbox{if}~ \sum_{e\in\delta(v)} y_e= 1\\
0&\mbox{otherwise}
\end{cases},\\
&\psi_{C}(y_{\delta(v_C)})= 
\begin{cases}
1&\mbox{if}~ \sum_{u\in V(C)}(-1)^{d_C(u,e)}y_{(v_C,u)}\in\{0,2\}\\
&\quad\sum_{e\in\delta(v_C)}y_e\le |C|-1\\
0&\mbox{otherwise}
\end{cases}.
\end{align*}
For the above GM \eqref{gm:oddmatching}, we derive the following corollary of Theorem \ref{thm:main}.
\begin{corollary}\label{cor:oddmatching}
If the solution of the maximum weight perfect matching with odd cycles LP \eqref{lp:oddmatching} is unique and integral, then 
the max-product BP on GM \eqref{gm:oddmatching} converges in $O(w_{\max}|E|^{2.5}/\rho)$ iterations. 
\end{corollary}
\begin{proof}
The proof of Corollary \ref{cor:oddmatching} can be done by using Theorem \ref{thm:main}.
From GM \eqref{gm:oddmatching}, each variable is connected to two factors ({\em C2} of Theorem \ref{thm:main}). Now, lets check {\em C3} of Theorem \ref{thm:main}.
For $v\in V$, we can apply same argument as the maximum weight matching case.
Suppose there are $v_C$ and $y_{\delta(v_C)}$ with $\psi_C(y_{\delta(v_C)})=1$.
Consider the case when there is ${(u_1,v_C)}\in\delta(v_C)$ with $y_{(u_1,v_C)}=1\ne y^*_{(u_1,v_C)}$. 
As a feasible solution $y_{\delta(v_C)}$ forms a disjoint even paths \citep{shin2013graphical},
check edges along the path contains $u_1$. If there is $u_2\in V(C)$ in the path with $y_{(u_2,v_C)}=1\ne y^*_{(u_2,v_C)}$ exists, choose such ${(u_1,v_C)}$.
If not, choose $(u_2,v_C)\in V(C)$ with $y_{(u_2,v_C)}=0\ne y^*_{(u_2,v_C)}$ at the end of the path.
On the other hand, consider the case when there is ${(u_1,v_C)}\in\delta(v_C)$ with $y_{(u_1,v_C)}=0\ne y^*_{(u_1,v_C)}$. 
As a feasible solution $y_{\delta(v_C)}$ form a disjoint even paths,
check edges along the path contains $u_1$. If there is $u_2\in V(C)$ in the path with $y_{(u_2,v_C)}=0\ne y^*_{(u_2,v_C)}$ exists, choose such ${(u_1,v_C)}$.
If not, choose $(u_2,v_C)\in V(C)$ with $y_{(u_2,v_C)}=1\ne y^*_{(u_2,v_C)}$ at the end of the path.
Then, from disjoint even paths point of view, we can check that
\begin{align*}
&\psi_C(y^\prime_{\delta(v_C)})=1,\\
&\qquad\mbox{
where
$y^\prime_{(u,v_C)} = \begin{cases}
		y_{(u,v_C)}~&\mbox{if}~u\ne u_1,u_2\\
		y^*_{(u,v_C)}~&\mbox{otherwise}
\end{cases}$.}\\
&\psi_C(y^{\prime\prime}_{\delta(v_C)})=1,\\
&\qquad\mbox{
where
$y^{\prime\prime}_{(u,v_C)} = \begin{cases}
		y_{(u,v_C)}~&\mbox{if}~u=u_1,u_2\\
		y^*_{(u,v_C)}~&\mbox{otherwise}
\end{cases}$.}
\end{align*}
From Theorem \ref{thm:main}, we can conclude that
if the solution of LP \eqref{lp:oddmatching} is unique and integral, the max-product BP on GM \eqref{gm:oddmatching} converges to the solution of LP \eqref{lp:oddmatching} in $O(w_{\max}|E|^{2.5}/\rho)$ iterations.
This completes the proof of Corollary \ref{cor:oddmatching}.
\end{proof}
We again emphasize that
a similar result for the maximum weight (non-perfect) matching problem was established in \citep{shin2013graphical}.
However, the proof technique in the paper does not extend to the perfect matching problem. This is in essence because
presumably the perfect matching problem is harder than the non-perfect matching one.
Under the proposed generic criteria of Theorem \ref{thm:main}, we overcome the technical difficulty.

\subsection{Example IV: Vertex Cover}\label{subsec:vc}
Given an undirected graph $G=(V,E)$ and non-negative integer vertex weights $b=[b_v:v\in V]\in \mathbb{Z}_+^{|V|}$, the vertex cover problem is to find a set of vertices
minimizes the sum of vertex weights in the set such that each edge is connected to at least one vertex in it.
This problem is one of {Karp's 21 NP-complete problems} \citep{karp1972reducibility}.
The associated LP formulation to the vertex cover problem is as follows:
\begin{equation}\label{lp:vertexcover0}
\begin{split}
	&\mbox{minimize}\qquad~ b\cdot y\\
	&\mbox{subject to}\qquad y_u+y_v\ge 1\\
	&\qquad\qquad\qquad~ y=[y_v]\in [0,1]^{|V|}.
\end{split}
\end{equation}
However, if we design a GM from the above LP,
it does not satisfy conditions in Theorem \ref{thm:main}.
Instead, we will show that BP can solve the following dual LP:
\begin{equation}\label{lp:vertexcover}
\begin{split}
	&\mbox{maximize}\qquad~ \sum_{e\in E}x_e\\
	&\mbox{subject to}\qquad \sum_{e\in\delta(v)} x_e\le b_v\\
	&\qquad\qquad\qquad~ x=[x_e]\in \mathbb R_+^{|E|}.
\end{split}
\end{equation}
Note that the above LP always has a half-integral solution.
As we did in Section \ref{subsec:matching},
one can duplicate edges, i.e., $E^\prime=\{e_1,\dots,e_{2b_{\max}}:e\in E\}$ with
$b_{\max}=\max_v b_v$, and design the following equivalent LP having an integral solution:
\begin{equation}\label{lp:vertexcover2}
\begin{split}
	&\mbox{maximize}\qquad~ w^\prime\cdot x\\
	&\mbox{subject to}\qquad \sum_{e_i\in\delta(v)} x_{e_i}\le 2b_v,\quad\forall\, v\in V\\
	&\qquad\qquad\qquad~ x=[x_{e_i}]\in [0,1]^{|E^\prime|}
\end{split},
\end{equation}
where $w^\prime_{e_i}=w_e$ for $e\in E$ and its copy $e_i\in E^{\prime}$.
From the above LP, we can construct the following GM:
\begin{equation}\label{gm:vertexcover}
	\Pr[X=x]~\propto~\prod_{e_i\in E^\prime} e^{w_{e_i}^\prime x_{e_i}}\prod_{v\in V} \psi_{v} (x_{\delta(v)}),
\end{equation}
where the factor function $\psi_v$ is defined as
\begin{align*}
&\psi_{v}(x_{\delta(v)}) = 
\begin{cases}
1&\mbox{if}~ \sum_{e_i\in\delta(v)} x_{e_i}\le2 b_v\\
0&\mbox{otherwise}
\end{cases}.
\end{align*}
For the above GM \eqref{gm:vertexcover}, one can easily check that Condition {\em C2} 
of Theorem \ref{thm:main} and the condition of Lemma \ref{lem:c3} hold. This directly leads to
the following corollary.
\begin{corollary}\label{cor:vertexcover}
If the solution of the vertex cover dual LP \eqref{lp:vertexcover2} is unique, then the
max-product BP on GM \eqref{gm:vertexcover} converges in $O(w_{\max}|E|^{2.5}/\rho)$ iterations. 
\end{corollary}
Again, the uniqueness condition in the above corollary is easy to guarantee by adding small random noises to edge weights $[w^\prime_{e_i}]$.
We further remark that if 
the solution of the primal LP \eqref{lp:vertexcover0} is integral, then
it can be easily found from the solution of the dual LP \eqref{lp:vertexcover2}
using the strictly complementary slackness condition \citep{bertsimas1997introduction} .

\subsection{Example V: Edge Cover}\label{subsec:ec}
Given an undirected graph $G=(V,E)$ and non-negative edge
weights $w=[w_e:e\in E]\in \mathbb{R}_+^{|E|}$ on edges, the minimum weight edge cover problem is to find a set of edges such that
each vertex is connected to at least one edge in the set and the sum of edge weights in the set is minimized. One can naturally design
the following LP for this problem:
\begin{equation}\label{lp:ec}
\begin{split}
	&\mbox{minimize}\qquad~ w\cdot x\\
	&\mbox{subject to}\qquad \sum_{e\in\delta(v)} x_e\ge 1\\
	&\qquad\qquad\qquad~ x=[x_e]\in [0,1]^{|E|}.
\end{split}
\end{equation}
where $\delta(v)$ is the set of edges connected to a vertex $v$. 
If the above LP has an integral solution, it corresponds to the solution of the minimum weight edge cover problem.

Similarly as the case of matching,
it is known that 
the minimum weight edge cover LP \eqref{lp:ec}
always has a half-integral solution $x^*\in \{0,\frac12,1\}^{|E|}$. We will design BP for obtaining the half-integral solution.
To this end, duplicate each edge $e$ to $e_1,e_2$ and define a new graph $G^\prime=(V,E^\prime)$ where $E^\prime=\{e_1,e_2:e\in E\}$.
Then, we suggest the following equivalent LP that always have an integral solution:
\begin{equation}\label{lp:modec}
\begin{split}
	&\mbox{minimize}\qquad~ w^\prime\cdot x\\
	&\mbox{subject to}\qquad \sum_{e_i\in\delta(v)} x_{e_i}\ge 2\\
	&\qquad\qquad\qquad~ x=[x_{e_i}]\in [0,1]^{|E^\prime|}.
\end{split}
\end{equation}
where $w^\prime_{e_1}=w^\prime_{e_2}=w_e$. 
One can easily observe that solving LP \eqref{lp:modec} is equivalent to solving LP \eqref{lp:ec}
due to our construction of $G^\prime$ and $w^\prime$.
Now, construct the following GM for LP \eqref{lp:modec}: 
\begin{equation}\label{gm:modec}
	\Pr[X=x]~\propto~\prod_{e_i\in E^\prime} e^{-w^\prime_{e_i} x_{e_i}}\prod_{v\in V} \psi_{v} (x_{\delta(v)}),
\end{equation}
where the factor function $\psi_v$ is defined as
\begin{align*}
&\psi_{v}(x_{\delta(v)}) = 
\begin{cases}
1&\mbox{if}~ \sum_{e_i\in\delta(v)} x_{e_i}\ge 2\\
0&\mbox{otherwise}
\end{cases}.
\end{align*}
For the above GM \eqref{gm:modec},
one can easily check that Condition {\em C2} 
of Theorem \ref{thm:main} and the condition of Lemma \ref{lem:c3} hold. This directly leads to
the following corollary.
\begin{corollary}\label{cor:ec}
If the solution of the minimum weight edge cover LP \eqref{lp:modmatching} is unique, then 
the max-product BP on GM \eqref{gm:modec} converges in $O(w_{\max}|E|^{2.5}/\rho)$ iterations. 
\end{corollary}
Again, the uniqueness condition in the above corollary is easy to guarantee by adding small random noises to edge weights $[w^\prime_{e_i}]$.

\subsection{Example VI: Traveling Salesman}\label{sec:tsp}
Given a directed graph $G=(V,E)$ and non-negative edge weights $w=[w_e:e\in E]\in \mathbb{R}_+^{|E|}$, 
the traveling salesman problem (TSP) is to find the minimum weight Hamiltonian cycle in $G$. 
The natural LP formulation to TSP is the following:
\begin{equation}\label{lp:tsp}
\begin{split}
	&\mbox{minimize}\qquad w\cdot x\\
	&\mbox{subject to}\qquad \sum_{e\in\delta(v)} x_{e}=2\\
	&\qquad\qquad\quad x=[x_{e}]\in [0,1]^{|E|}.
\end{split}
\end{equation}
From the above LP, one can construct the following GM:
\begin{equation}\label{gm:tsp}
	\Pr[X=x]~\propto~\prod_{e\in E} e^{-w_e x_e}\prod_{v\in V} \psi_{v} (x_{\delta(v)}),
\end{equation}
where the factor function $\psi_v$ is defined as
\begin{align*}
&\psi_{v}(x_{\delta(v)}) = 
\begin{cases}
1&\mbox{if}~ \sum_{e\in\delta(v)} x_{e}=2\\
0&\mbox{otherwise}
\end{cases}.
\end{align*}
It is known that LP \eqref{lp:tsp} always has an integral solution \cite{bolker2006simple}.
For the above GM \eqref{gm:tsp}, one can easily check that Condition {\em C2} 
of Theorem \ref{thm:main} and the condition of Lemma \ref{lem:c3} hold. This directly leads to
the following corollary.
\begin{corollary}\label{cor:tsp}
If the solution of the traveling salesman LP \eqref{lp:tsp} is unique, 
then the max-product BP on GM \eqref{gm:tsp} converges in $O(w_{\max}|E|^{2.5}/\rho)$ iterations. 
\end{corollary}
Again, the uniqueness condition in the above corollary is easy to guarantee by adding small random noises to edge weights. 
\subsection{Example VII: Maximum Weight Cycle Packing}
Given an undirected graph $G=(V,E)$ and non-negative edge weights $w=[w_e:e\in E]\in \mathbb{R}_+^{|E|}$, 
the maximum weight vertex disjoint cycle packing problem is to find the maximum weight set of cycles with no common vertex.
It is easy to observe that it is equivalent to find
a subgraph maximizing the sum of edge weights on it such that each vertex of the subgraph has degree 2 or 0.
The natural LP formulation to this problem is following:
\begin{equation}\label{lp:maxvertexpacking}
\begin{split}
	&\mbox{maximize}\qquad w\cdot x\\
	&\mbox{subject to}\qquad \sum_{e\in\delta(v)} x_{e}=2y_v\\
	&\qquad\qquad x=[x_{e}]\in [0,1]^{|E|},y=[y_v]\in[0,1]^{|V|}.
\end{split}
\end{equation}
From the above LP, one can construct the following GM:
\begin{equation}\label{gm:maxvertexpacking}
	\Pr[X=x,Y=y]~\propto~\prod_{e\in E} e^{w_e x_e}\prod_{v\in V} \psi_{v} (x_{\delta(v)},y_v),
\end{equation}
where the factor function $\psi_v$ is defined as
\begin{align*}
&\psi_{v}(x_{\delta(v)},y_v) = 
\begin{cases}
1&\mbox{if}~ \sum_{e\in\delta(v)} x_{e}=2y_v\\
0&\mbox{otherwise}
\end{cases}.
\end{align*}
For the above GM \eqref{gm:maxvertexpacking}, one can easily check that Condition {\em C2} 
of Theorem \ref{thm:main} and the condition of Lemma \ref{lem:c3} hold. This directly leads to
the following corollary.
\begin{corollary}\label{cor:maxvertexpacking}
If the solution of maximum weight vertex disjoint cycle packing LP \eqref{lp:maxvertexpacking} is unique and integral, 
then the max-product BP on GM \eqref{gm:maxvertexpacking} converges in $O(w_{\max}|E|^{2.5}/\rho)$ iterations. 
\end{corollary}
Again, the uniqueness condition in the above corollary is easy to guarantee by adding small random noises to edge weights. 

\subsection{Example VIII: Minimum Cost Network Flow}
Given a directed graph $G=(V,E)$, supply/demand $d=[d_v]\in\mathbb{Z}^{|E|}$ and 
capacity $c=[c_e:e\in E]\in\mathbb{Z}_+^{|E|}$,
the minimum cost network flow problem can be forumlated by the following LP.

\begin{equation}\label{lp:networkflow}
\begin{split}
	&\mbox{minimize}\qquad~ w\cdot x\\
	&\mbox{subject to}\qquad \sum_{e\in \delta^o(v)} x_{e}-\sum_{e\in\delta^i(v)} x_{e}=d_v\\
	&\qquad\qquad\qquad~ x_e\le c_e\\
	&\qquad\qquad\qquad~ x=[x_{e}]\in \mathbb{R}_+^{|E|},
\end{split}
\end{equation}
where $\delta^i(v),\delta^o(v)$ are the set of incoming, outgoing edges of $v$.
It is known that the above LP always has an integral solution.
We will design BP for obtaining the solution of LP \eqref{lp:networkflow}.
To this end, duplicate each edge $e$ to $e_1,\dots,e_{c_e}$ and define a new graph $G^\prime=(V,E^\prime)$ where $E^\prime=\{e_1,\dots,e_{c_e}:e\in E\}$.
Then, we suggest the following equivalent LP that always have an integral solution:
\begin{equation}\label{lp:networkflow2}
\begin{split}
	&\mbox{minimize}\qquad~ w^\prime\cdot x\\
	&\mbox{subject to}\qquad \sum_{e_i\in \delta^o(v)}x_{e_i}-\sum_{e_i\in\delta^i(v)} x_{e_i}=d_v\\
	&\qquad\qquad\qquad~ x=[x_{e_i}]\in[0,1]^{|E^\prime|}.
\end{split}
\end{equation}
where $w^\prime_{e_1}=\dots=w^\prime_{e_{c_e}}=w_e$. One can easily observe that solvin
LP \eqref{lp:networkflow} is equivalent to solving LP \eqref{lp:networkflow2} due to our construction
of $G^\prime$ and $w^\prime$.
Now, construct the following GM for LP \eqref{lp:networkflow2}:
\begin{equation}\label{gm:networkflow}
	\Pr[X=x]~\propto~\prod_{e_i\in E^\prime}e^{-w^\prime_{e_i} x_{e_i}}\prod_{v\in V} 
	\psi_{v} (x_{\delta(v)}),
\end{equation}
where the factor function $\psi_v$ is defined as
\begin{align*}
&\psi_{v}(x_{\delta(v)}) = 
\begin{cases}
1&\mbox{if}~ \sum_{e_i\in\delta^o(v)} x_{e_i}
-\sum_{e_i\in\delta^i(v)} x_{e_i}=d_v\\
0&\mbox{otherwise}
\end{cases}.
\end{align*}
For the above GM \eqref{gm:networkflow}, one can easily check that Condition {\em C2} 
of Theorem \ref{thm:main} and the condition of Lemma \ref{lem:c3} hold. This directly leads to
the following corollary.
\begin{corollary}\label{cor:networkflow}
If the solution of the network flow LP \eqref{lp:networkflow} is unique, then the max-product BP on GM \eqref{gm:networkflow} converges in $O(w_{\max}|E^\prime|^{2.5}/\rho)$ iterations.
\end{corollary}
Gamarnik et al. \cite{gamarnik2012belief} also studied the convergence and correct of BP on the minimum cost network flow problem. 
However, they studied BP on GM of continuous variables while our analysis is for BP on GM
of binary variables.
For practical purposes, the latter is easier to run than the former.

\section{Proof of Theorem \ref{thm:main}}\label{sec:mainpf}
To begin with, we define some necessary notation.
We let $\mathcal P$ denote the polytope of feasible solutions of LP \eqref{eq:lp1}:
$$\mathcal P :=\left\{x\in [0,1]^n\,:\, \psi_\alpha(x_\alpha)=1,~\forall\,\alpha\in F \right\}.$$
Similarly, $\mathcal P_\alpha$ is defined as
$$\mathcal P_\alpha :=\left\{x\in [0,1]^{|\alpha|}\,:\, \psi_\alpha(x_\alpha)=1 \right\}.$$
Now, we state the following key technical lemma.
\begin{lemma}\label{lemma:c4}
There exist universal constants $\eta>0$ for LP \eqref{eq:lp1} such that
if $z\in[0,1]^n$ and $0<\varepsilon<\eta$ satisfy the followings:
\begin{itemize}
\item[P1.]
There exist at most two violated factors for $z$, i.e., 
$\left|\{\alpha\in F\,:\, z_\alpha\notin \mathcal{P}_\alpha\}\right|\leq 2.$

\item[P2.] For each violated factor $\alpha$,
there exists $i\in \alpha$ such that
$z^\dagger_\alpha\in \mathcal P_\alpha,$ 
where $z^\dagger= z + \varepsilon e_i$
or $z^\dagger = z - \varepsilon e_i$ where
$e_i\in\{0,1\}^n$ is the unit vector whose $i$-th coordinate is $1$,
\end{itemize}
then there exists $z^\ddagger\in \mathcal P$ such that $\|z-z^\ddagger\|_1 \leq \varepsilon K$.\footnote{
$K$ is defined in Theorem \ref{thm:main}.} 
\end{lemma}
The proof of Lemma \ref{lemma:c4} is presented in Section \ref{sec:pflemma:c4}.
Now, from Condition {\em C1}, it follows that 
there exists $\rho>0$ such that
	\begin{equation*}
		\rho:=\inf_{x\in\mathcal{P}\setminus x^*} \frac{w\cdot x- w\cdot x^*}{\|x-x^*\|_1}>0.
	\end{equation*}
We let $\hat{x}^t\in\{0,1,?\}^n$ denote the BP estimate at the $t$-th iteration for the MAP computation.
We will show that under Conditions {\em C1-C3}, 
$$\hat{x}^t = x^*,\qquad\mbox{for}\quad
t > \left(\frac{w_{\max}}\rho+1\right)K.$$

Suppose the above statement is false, i.e., there exists $i\in\{1,2,\dots, n\}$ such that
$\hat{x}^t_i \neq x^*_i$ for $t > \left(\frac{w_{\max}}\rho+1\right)K$. 
Under the assumption, we will reach a contradiction.
To this end, we construct a tree-structured GM $T_i(t)$, popularly known as the computational tree \citep{weiss2001optimality}, as follows: 
\begin{itemize}
\item[1.] Add $y_i\in \{0,1\}$ as the root variable with variable factor function $e^{-w_i y_i}$.
\item[2.] 
{For each leaf variable $y_j$, for each $\alpha\in F_j$ and $\psi_\alpha$  which is not associated with $y_j$ in the current tree-structured GM, add  
a factor function $\psi_\alpha$ as a child of $y_j$.}
\item[3.] 
{For each leaf factor $\psi_\alpha$, for each variable $y_k$ such that $k\in\alpha$ and $y_k$ is not associated with $\psi_\alpha$ in 
the current tree-structured GM, add a variable $y_k$ as a child of $\psi_\alpha$ with variable factor function $e^{-w_k y_k}$.}
\item[4.] Repeat Step 2, 3 $t$ times.
\end{itemize}
Suppose the initial messages of BP are set by 1, i.e., $m_{j\to\alpha}(\cdot)^0 =1$.
Then, if $\hat{x}^t_i\in\{0,?\}$, it is known \citep{weiss1997belief} 
that there exists a MAP configuration 
$y^{\tt MAP}$ on $T_i(t)$ with $y^{\tt MAP}_i=0$ at
the root variable. A similar conclusion also
holds for the case $\hat{x}^t_i\in\{1,?\}$.
For other initial messages, one can guarantee the same property under
changing weights of leaf variables of the tree-structured GM.
Specifically, 
for a leaf variable $k$ with $|F_k=\{\alpha_1,\alpha_2\}|=2$ and $\alpha_1$
being its parent factor in $T_i(t)$,
one can reset its variable factor by $e^{-w_k^\prime y_k}$, where
\begin{equation}
w^\prime_{k}=w_k-\log\frac{\max_{z_{\alpha_2}:z_k=1}\psi_{\alpha_2}(z_{\alpha_2})\Pi_{j\in\alpha_2\setminus k}m_{j\rightarrow\alpha_2}^0(z_j)}
{\max_{z_{\alpha_2}:z_k=0}\psi_{\alpha_2}(z_{\alpha_2})\Pi_{j\in\alpha_2\setminus k}m_{j\rightarrow\alpha_2}^0(z_j)}.\label{eq:otherinitial}
\end{equation}
This is the reason why our proof of Theorem \ref{thm:main} goes through for arbitrary initial messages.
For notational convenience, we present the proof for the standard initial message of $m_{j\to\alpha}^0(\cdot) =1$,
where it can be naturally generalized to other initial messages using \eqref{eq:otherinitial}.

Now we construct a new valid assignment $y^{\tt NEW}$ on the computational tree $T_i(t)$
as follows:
\begin{itemize}
\item[1.] Initially, set $y^{\tt NEW} \leftarrow y^{\tt MAP}$.
\item[2.] Update the value of the root variable of $T_i(t)$ by $y^{\tt NEW}_i\leftarrow x^*_i$.
\item[3.] For each child factor $\psi_\alpha$ of root $i\in\alpha$, choose $\gamma\subset\alpha$ according to Condition {\em C3}
and update the associated variable by $y^{\tt NEW}_{j}\leftarrow x^*_{j}~~\forall j\in\gamma$.
\item[4.] Repeat Step 2,3 recursively by substituting $T_i(t)$ by the subtree of $T_i(t)$ of root $j\in\gamma$ until the process stops ({i.e., $\gamma=\{i\}$}) or the leaf of $T_i(t)$ is reached
(i.e., $i$ does not have a child).
\end{itemize}
One can notice that the set of revised variables in Step 2 of the above procedure forms a path structure $Q$ in the tree-structured GM. 
Define 
$\zeta_j$ and $\kappa_j$ be the number of copies of $x_j$ in path $Q$ with 
$x^*_{j}=1$ and $x^*_{j}=0$, respectively, where
$\zeta=[\zeta_j] ,\kappa=[\kappa_j]\in \mathbb Z_+^{n}$ .
Then, from our construction of $y^{\tt NEW}$, one can observe that 
\begin{align*}
&w \cdot 
y^{\tt MAP}-w\cdot
y^{\tt NEW}= w\cdot (\kappa-\zeta).
\end{align*}

We consider three cases:
(a) no end of the path $Q$ touches a leaf of $T_i(t)$,
(b) only one end of the path $Q$ touches a leaf of $T_i(t)$,
and (c)
both ends of the path $Q$ touch leaves of $T_i(t)$.
First, consider the case (a).
If we set $z=x^*+\varepsilon(\kappa-\zeta)$ where $0<\varepsilon<\frac1{2t}$, then due to
our construction of $y^{\tt NEW}$ utilizing Condition {\em C3}, one can observe $z\in\mathcal{P}$.
However, since $x^*$ is the unique optimum of LP \eqref{eq:lp1}, we have
$$w\cdot y^{\tt MAP}-w\cdot y^{\tt NEW}=\frac{1}{\varepsilon}(w\cdot z-w\cdot x^*)>0,$$
which contradicts to the fact that $y^{\tt MAP}$ is a MAP configuration. 
Next, consider the case (c), where the case (b) can be argued in a similar manner.
In this case, we use Lemma \ref{lemma:c4} by setting
$z=x^*+\varepsilon(\kappa-\zeta)$ where $0<\varepsilon<\min\left\{\frac1{2t},\eta\right\}$ and
one can check that $z$ satisfies Conditions {\em P1, P2} of Lemma \ref{lemma:c4} due to
Conditions {\em C2, C3}.
Hence, from Lemma \ref{lemma:c4},  
there exists $z^\ddagger \in \mathcal P$ such that
\begin{align*}
&\|z^\ddagger - z\|_1 \leq \varepsilon K\quad\mbox{and}\quad
\|z^\ddagger - x^*\|_1 \geq \varepsilon (\|\zeta\|_1+\|\kappa\|_1 - K) \geq \varepsilon (t-K).
\end{align*}
Hence, it follows that
\begin{equation*}
\begin{split}
	0<\rho&\le\frac{w\cdot z^\ddagger-w\cdot x^*}{\|z^\ddagger-x^*\|_1}\\
	&\le\frac{w\cdot z+\varepsilon w_{\max} K-w\cdot x^*}{\varepsilon(t-K)}\\
	&=\frac{\varepsilon w\cdot (\kappa-\zeta)+\varepsilon w_{\max}K}{
\varepsilon(t-K)}\\
&=\frac{w\cdot(\kappa-\zeta)+w_{\max}K}{t-K}.
\end{split}
\end{equation*}
Furthermore, if $t > \left(\frac{w_{\max}}{\rho}+1\right)K$, the above inequality implies that
\begin{align*}
	w\cdot y^{\tt MAP}-w\cdot y^{\tt NEW}
&=w\cdot (\kappa-\zeta)\\
&\ge \rho t -(w_{\max} +\rho) K~>~0.
\end{align*}
This is the contradiction to the fact that
$y^{\tt MAP}$ is a MAP configuration. 
This completes the proof of Theorem \ref{thm:main}.
\subsection{Proof of Lemma \ref{lemma:c4}}\label{sec:pflemma:c4}
We first define $\mathcal{P}_{\varepsilon}=\{x\,:\,Ax\ge b-\varepsilon\mathbf{1}\}$,
where $\mathbf{1}$ is the vector of ones.
Then, one can check that $z\in\mathcal{P}_\varepsilon$ for $z,\varepsilon$ satisfying conditions of Lemma \ref{lemma:c4}. 
Now we aim to achieve the following inequality
$${{\tt dist}}(\mathcal{P},\mathcal{P}_\varepsilon):=\max_{x\in\mathcal{P}_\varepsilon}\min_{y\in\mathcal{P}} \|x-y\|_1 \le \varepsilon K,$$
which leads to the conclusion of  Lemma \ref{lemma:c4}.
To this end, for
$\xi\subset[1,2,\dots, m]$ with $|\xi|=n$,
we again let $A_\xi$ be
the square sub-matrix of $A$ by choosing
$\xi$-th rows of $A$ and $b_\xi$ is the $n$-dimensional subvector of $b$ corresponding $\xi$.
Using this notation, we first prove the following claim.
\begin{claim}\label{clm1}
If $A_{\xi}$ is invertible and $v_\xi:=(A_{\xi})^{-1} b_{\xi}\in \mathcal P$, then
$v_\xi$ is a vertex of 
polytope $\mathcal{P}$.
\end{claim}
\begin{proof}
Suppose $v_\xi$ is not a vertex of $\mathcal P$, 
i.e. there exist $x,y\in\mathcal{P}$ such that $x\neq y$ and
$v_\xi=\lambda x+(1-\lambda)y$ for some $\lambda\in(0,1/2]$.
Under the assumption, we will reach a contradiction.
Since $\mathcal{P}$ is a convex set, 
\begin{equation}
\frac{3\lambda}2x+\left(1-\frac{3\lambda}2\right)y\in\mathcal{P}.\label{eq1:pfclm1}
\end{equation}
However, as $A_\xi$ is invertible, 
\begin{equation}
A_\xi\left(\frac{3\lambda}2x+\left(1-\frac{3\lambda}2\right)y\right)\ne b_\xi.\label{eq2:pfclm2}
\end{equation}
From \eqref{eq1:pfclm1} and \eqref{eq2:pfclm2}, there exists a row vector $A_{i*}$ of $A_\xi$
and the corresponding entry $b_i$ of $b_\xi$ such that 
$$A_{i*}\cdot \left(\frac{3\lambda}2x+\left(1-\frac{3\lambda}2\right)y\right)>b_i.$$
Using the above inequality and $A_{i*}\cdot (\lambda x+(1-\lambda)y)=b_i,$ one can conclude that
$$A_{i*}\cdot \left(\frac{\lambda}2 x+\left(1-\frac{\lambda}2\right)y\right)<b_i,$$
which contradict to $\frac{\lambda}2x+\left(1-\frac{\lambda}2\right)y\in\mathcal{P}$.
This completes the proof of Claim \ref{clm1}.
\end{proof}
We also note that if $v$ is a vertex of polytope $\mathcal P$, there exists $\xi$ such that
$A_{\xi}$ is invertible and $v=(A_{\xi})^{-1} b_{\xi}$.
We define the following notation: 
\begin{align*}
&
\mathcal I=\{\xi\,:\, (A_\xi)^{-1}b_\xi\in\mathcal{P}\}
\quad\mbox{and}\quad
\mathcal I_{\varepsilon}=\{\xi\,:\,(A_\xi)^{-1}(b_\xi-\varepsilon\mathbf{1})\in\mathcal{P}_\varepsilon\},
\end{align*} 
where Claim \ref{clm1} implies that
$\{v_\xi:=(A_{\xi})^{-1} b_{\xi}\,:\,\xi\in \mathcal I\}$ and $\{u_{\xi,\varepsilon} :=(A_\xi)^{-1}(b_\xi-\varepsilon\mathbf{1})\,:\,\xi\in \mathcal I_{\varepsilon}\}$
are sets of vertices of $\mathcal{P}$ and $\mathcal{P}_\varepsilon$, respectively.
Using the notation, we show the following claim. 
\begin{claim}\label{claim:c4}
There exists $\eta>0$ such that $\mathcal I_{\varepsilon}\subset \mathcal I$
for all $\varepsilon\in (0,\eta)$.
\end{claim}
\begin{proof}
Suppose $\eta>0$ satisfying the conclusion of Claim \ref{claim:c4} does not exist.
Then, there exists a strictly decreasing sequence $\{\varepsilon_k>0:k=1,2,\dots\}$ converges to 0 such that
$\mathcal{I}_{\varepsilon_k}\cap\{\xi\,:\,\xi\notin\mathcal{I}\}\neq\emptyset .$
Since $|\{\xi:\xi\subset [1,2,\dots, m]\}|<\infty$, there exists $\xi^\prime$ such that 
\begin{equation}
|\mathcal K:=\{k\,:\,\xi^\prime\in\mathcal{I}_{\varepsilon_k}\cap\{\xi\,:\,\xi\notin\mathcal{I}\}\}|=\infty .\label{eq1:pfclaimc4}
\end{equation}
For any $k\in \mathcal K$,
observe that
the sequence $\{u_{\xi^\prime,\varepsilon_\ell}:\ell\geq k,\ell\in \mathcal K\}$ converges to $v_{\xi^\prime}$.
Furthermore, all points in the sequence are in $\mathcal{P}_{\varepsilon_k}$
since  $\mathcal{P}_{\varepsilon_\ell}\subset\mathcal{P}_{\varepsilon_k}$ for any $\ell\geq k$. 
Therefore,
one can conclude that $v_{\xi^\prime}\in\mathcal{P}_{\varepsilon_k}~\mbox{for all}~k\in \mathcal K,$
where we additionally use the fact that $\mathcal{P}_{\varepsilon_k}$ is a closed set.
Because $\mathcal{P}=\bigcap_{k\in \mathcal K}\mathcal{P}_{\varepsilon_k}$, it must be that $v_{\xi^\prime}\in\mathcal{P}$, i.e.,
$v_{\xi^\prime}$ must be a vertex of $\mathcal{P}$ from Claim \ref{clm1}.
This contradicts to the fact
$\xi^\prime\in\{\xi\,:\,\xi\notin\mathcal{I}\}$.
This completes the proof of Claim \ref{claim:c4}.
\end{proof}
From the above claim, we observe that
any $x\in\mathcal{P}_\varepsilon$ can be expressed as a convex combination of $\{u_{\xi,\varepsilon}\,:\,\xi\in \mathcal{I}\}$,
i.e., $x=\sum_{\xi\in \mathcal{I}}\lambda_\xi u_{\xi,\varepsilon}$ with $\sum_{\xi\in \mathcal{I}}\lambda_\xi =1$ and $\lambda_\xi\geq 0$.
For all $\varepsilon\in(0,\eta)$ for $\eta>0$ in Claim \ref{claim:c4}, one can conclude that
\begin{align*}
{\tt dist}(\mathcal{P},\mathcal{P}_\varepsilon)
&\le\max_{x\in\mathcal{P}_\varepsilon}\|\sum_{\xi\in\mathcal{I}}\lambda_\xi u_{\xi,\varepsilon}-\sum_{\xi\in\mathcal{I}}\lambda_\xi v_\xi\|_1\\
&=\max_{x\in\mathcal{P}_\varepsilon}\varepsilon\|\sum_{\xi\in\mathcal{I}}\lambda_\xi (A_\xi)^{-1}\mathbf{1}\|_1\\
&\le\varepsilon \max_{\xi\in\mathcal{I}} \|(A_\xi)^{-1}\mathbf{1}\|_1\\
&\le \varepsilon K.
\end{align*}
This completes the proof of Lemma \ref{lemma:c4}.

\section{Conclusion}

The BP algorithm has been the most popular algorithm for solving inference problems arising graphical models,
where its distributed implementation, associated ease of programming
and strong parallelization potential are the main reasons for its growing popularity.
In this paper, we aim for designing BP algorithms solving LPs,
and provide sufficient conditions for its correctness and convergence.
We believe that our results provide
new interesting directions on designing efficient distributed (and parallel) solvers for large-scale LPs.

\bibliographystyle{plain}
\bibliography{reference}

\end{document}